\documentclass[journal,11pt,doublespace,onecolumn]{IEEEtran}

\usepackage{times}
\usepackage{bm}

\usepackage[plain,noend]{algorithm2e}

\usepackage{amsfonts}
\usepackage{graphicx}
\usepackage{color,soul}
\usepackage{cite,balance}
\usepackage{bm}
\usepackage{booktabs}
\usepackage{flushend}
\usepackage{tikz}
\usetikzlibrary{arrows}
\usepackage{subfigure}
\RequirePackage{amsthm,amsmath,amsfonts,amssymb}
\usepackage{multirow}
\usepackage{bbm}

\newtheorem{theorem}{Theorem}[section]
\newtheorem{lemma}[theorem]{Lemma}

\theoremstyle{remark}

\newtheorem{proposition}[theorem]{Proposition}

\newtheorem{remark}{Remark}
\newtheorem{assumption}{Assumption}

\newcommand{\vw}{\bm w}
\newcommand{\vx}{{\bm{x}}}

\newcommand{\vz}{{\bm{z}}}

\newcommand{\vI}{{\bm{I}}}

\newcommand{\Bl}{{\Big |}}
\newcommand{\pr}{{\mathbb{P}}}

\newcommand{\var}{\mathrm{var}}

\definecolor{gen}{RGB}{0, 0, 255}

\def\R{\mathcal{R}}

\def\E{\mathbb{E}\,}
\def\de{\overset{\Delta}{=}}

\usepackage{amsmath}

\DeclareMathOperator*{\argmin}{arg\,min}
\def\T{{ \mathrm{\scriptscriptstyle T} }}
\def\i{\mathbbm{1}}
\def\P{\mathbb{P}\,}
\def\l{\textrm{loss}}
\def\Ltwo{\ell_2}
\def\Lone{\ell_1}

\begin{document}

\title{The Rate of Convergence of Variation-Constrained \\ Deep Neural Networks}

\author{Gen~Li and Jie~Ding 
\thanks{G.~Li is with the Department of Statistics and Data Science, University of Pennsylvania Wharton School. J.~Ding is with the School of Statistics, University of Minnesota Twin Cities.
}
}




\maketitle

\begin{abstract}
Multi-layer feedforward networks have been used to approximate a wide range of nonlinear functions. An important and fundamental problem is to understand the learnability of a network model through its statistical risk, or the expected prediction error on future data. To the best of our knowledge, the rate of convergence of neural networks shown by existing works is bounded by at most the order of $n^{-1/4}$ for a sample size of $n$. In this paper, we show that a class of variation-constrained neural networks, with arbitrary width, can achieve near-parametric rate $n^{-1/2+\delta}$ for an arbitrarily small positive constant $\delta$. It is equivalent to $n^{-1 +2\delta}$ under the mean squared error. This rate is also observed by numerical experiments. The result indicates that the neural function space needed for approximating smooth functions may not be as large as what is often perceived. Our result also provides insight to the phenomena that deep neural networks do not easily suffer from overfitting when the number of neurons and learning parameters rapidly grow with $n$ or even surpass $n$. We also discuss the rate of convergence regarding other network parameters, including the input dimension, network layer, and coefficient norm.

\end{abstract}


\section{Introduction} \label{sec_intro}

Suppose we have $n$ labeled observations $\{(\vx_i, y_i)\}_{i=1,\ldots,n}$, 
where $y_i$'s are continuously-valued responses or labels. 
We assume that the underlying data generating model is
$$
y_i = f_*(\vx_i) + \varepsilon_i 
$$
for some unknown function $f_*(\cdot)$,
where $\vx_i$'s $ \in \mathbb{X} \subset \mathbb{R}^d$ are independent and identically distributed (IID),
 and $\varepsilon_i$'s are IID.
We will measures the predictive performance of a learned neural networks model $f$ using the $\Ltwo$ statistical risk
\begin{align}
\R_2(f) \de \sqrt{\E (f(\vx) - f_*(\vx))^2}, \label{eq_l2risk}
\end{align}
where $\vx$ denotes an observation independent with the data used to train $f$. 
The smaller $\R_2(f)$, the better.
We will also study $\Lone$ risks, which will be introduced in Subsection~\ref{subsec_loss}. 
Our work addresses the following question.

\textit{How does a multi-layer feedforward network's predictive performance relate to the sample size used for training? }

\subsection{Background of deep neural networks (DNNs)}

Neural networks have been successfully applied to modeling nonlinear regression functions or classification decision boundaries in various applications. 
Despite its success in practical applications, a systematic understanding of its theoretical limit remains an ongoing challenge and has motivated research from various perspectives. 
It was shown in \cite{cybenko1989approximations} that any continuous function can be approximated arbitrarily well by a two-layer perceptron with sigmoid activation functions.
An approximation error bound of using two-layer neural networks to fit arbitrary smooth functions was established in \cite{barron1993universal,barron1994approximation}, where statistical risk bounds were also developed.
A theoretical connection between deep networks and approximation theory based on spline functions was established in \cite{Baraniuk}.
A dimension-free Rademacher complexity for deep ReLU neural networks was recently developed~\cite{golowich2017size,barron2019complexity}.
A dimension-free statistical risk for two-layer neural networks was studied in \cite{DingNN} using a combined analysis of Rademacher complexity and the number of neurons.
Based on a {contraction lemma}, 
a series of variation-based complexities and their corresponding generalization errors were developed~\cite{neyshabur2015norm}.
The neural network learning problem was also cast as a tensor decomposition problem through the score function of the known or estimated input distribution~\cite{janzamin2015beating,ge2017learning,mondelli2018connection}. 
Most recently, tight error bounds have been established for deep neural networks that do not suffer from the curse of dimensionality, based on the assumption that the data are generated by a neural network model of a parsimonious structure. 
In that direction, the work of \cite{schmidt2017nonparametric} proved that specific deep neural networks with few non-zero network parameters could achieve minimax rates of convergence. 
In \cite{bauer2019deep}, an error bound that is free from the input dimension was developed for a class of generalized hierarchical interaction models.

\subsection{Related results on nonparametric regression} \label{subsec_nonpar}
The $\Lone$ or $\Ltwo$ risks of typical parametric models such as finite-dimensional linear regressions are at the order of $n^{-1/2}$, where $n$ denotes the sample size.
It is known that the minimax rate of convergence for smoothness function classes (e.g., Sobolev and Besov) is often at the order of $O(n^{-\alpha/(d+2\alpha)})$, where $\alpha$ is the order of smoothness (defined by the largest order of derivatives) and $d$ is the input data dimension.
Obtaining the risk bound for general nonparametric regression models such as neural networks is highly nontrivial since the networks involve a large number of parameters to ensure the fitness ability (i.e., to make approximation error small enough).
The work of \cite{barron1993universal,barron1994approximation} proved that the model class of two-layer feedforward networks has an approximation error bound of $O(r^{-1/2})$ when approximating a particular class of smooth functions, where $r$ denotes the number of neurons. The same work further developed a statistical risk error bound of $O(n^{-1/4})$, which is among the tightest statistical risk bounds for two-layer neural networks up to the authors' knowledge.
A similar rate was also shown in~\cite{yang1999information}.
The above risk bound was derived based on an optimal bias-variance tradeoff that involves an appropriate choice of $r$. 
Also, to bound the statistical risk, it is tempting to treat all the neural weights as free parameters and use the Akaike information criterion-type statistical error bounds~\cite{DingOverview}. Those bounds were often derived from second-order Taylor expansion in conjunction with some regular conditions~\cite{DingLOL}.
Recently, variation-constrained neural networks attracted a lot of attention, as they allow for a vast number of parameters relative to the sample size.
Using Rademacher complexity as machinery, several recent works have established statistical risk bounds for DNNs under the norm constraints, which are at the order of either $n^{-1/4}$ or $n^{-1/6}$ (equivalently, the square root or cube root of our rate $n^{-1/2}$)~\cite{neyshabur2015norm,golowich2017size,barron2019complexity}.

\subsection{Main contributions}

We show that variation-constrained deep neural networks can achieve near-parametric rate $n^{-1/2+\delta}$, for an arbitrarily small positive constant $\delta$. 
The result has the following implications.
First, to learn smooth regression functions, where the smoothness is indexed by its variation, the class of deep neural networks can enjoy a statistical efficiency at the same order of classical parametric models (namely around $n^{-1/2}$). 
Second, the neural function space needed for approximating smooth functions may not be as large as what is often perceived. 
In particular, we will show that the minimax rate of the statistical risk is at least $n^{-1/2}$, which implies that our derived bound is rather tight.
Third, a neural network's predictive performance does not necessarily depend on its `nominal complexity' as described by its number of neurons and layers. Instead, it only depends on the property of the underlying function and $n$. To some extent, this explains the interesting phenomena that \textit{deep neural networks do not easily suffer from overfitting when the number of unknown parameters is huge compared with $n$}. 

From the technical perspective, we derived the risk analysis from a variation-based approach, which is different from earlier work that used the number of neurons $r$ to characterize bias-variance tradeoffs. In particular, the risk bound of $O(n^{-1/4})$ in \cite{barron1994approximation} is based on an appropriate choice of $r$, but our tighter bound $O(n^{-1/2+\delta})$ is based on an analysis of the neural weights regardless of $r$.   
Our analysis was inspired by but different from the norm-based complexity analysis~\cite{neyshabur2015norm,barron2019complexity}. 
In particular, the standard technical tool of contraction lemma used to derive Rademacher complexity cannot apply to establishing the $O(n^{-1/2+\delta})$ statistical risk studied in our context. Our proof is based on new technical analyses.
 
Apart from the statistical risk based on the $\Ltwo$ loss, we will analyze the $\Lone$ loss, which also ranks among the most popular loss functions in practice.
We will show that a similar convergence rate can be derived, and the dependence on the input dimension can be removed. 
To the best of the authors' knowledge, this is the first result regarding the $\Lone$ statistical risk for deep neural networks.
 
The outline of the paper is given below. 
In Section~\ref{sec_background}, we introduce the formulation and some notations.
In Section~\ref{sec_main}, we introduce the main result and make several remarks to illustrate the implications better. 
In Section~\ref{sec_L1}, we introduce counterpart results when the typical $\Ltwo$ loss is replaced with $\Lone$ loss and show how the bound as a function of the norm may be further tightened.
We conclude the work in Section~\ref{sec_conclusion}.

\section{Problem Formulation} \label{sec_background}

\subsection{Notation}

Throughout the paper, we use $n$ and $d$ to denote the sample size and the number of variables (or input dimension), respectively. 
We write $a_n \gtrsim b_n$, $b_n \lesssim a_n$, or $b_n=O(a_n)$, if $|b_n / a_n| < c$ for some constant $c$ for all sufficiently large $n$.  
Let $\mathcal{N}(\bm \mu, \bm \Sigma)$ denote Gaussian distribution with mean $\bm \mu$ and covariance $\bm \Sigma$.
Let $\|\cdot\|_1$ and $\|\cdot\|_2$ denote the common $\Lone$ and $\Ltwo$ vector norms, respectively. 
For any vector $\vz \in \mathbb{R}^d$ and set $\mathbb{X}$, we define $\|\vz\|_{\mathbb{X}} \de \sup_{\vx \in \mathbb{X}} |\vx^{\top}\vz|$, which may or may not be infinity.
Assume $\mathbb{X} = \{\vx : \|\vx\|_{\infty} \le M\}$ for some constant $M$ throughout the paper, and $\E$ denotes the expectation for the underlying data generating distribution. 
The notation for neural network parameters is summarized in Subsection~\ref{subsec_NNnotation}.

\subsection{Data generating process}

Suppose we have $n$ observations $\{(\vx_i, y_i)\}_{i=1,\ldots,n}$, where $y_i$'s are continuously-valued responses or labels. 
We assume that the data are independently generated from the underlying data generating process
$$
y_i = f_*(\vx_i) + \varepsilon_i 
$$
where $f_*(\cdot)$ is an unknown function. We will use a neural network to approximate the functionality of $f_*(\cdot)$.
We also assume that $\vx_i$'s $ \in \mathbb{X}$ are IID variables, where $\mathbb{X} \subset \mathbb{R}^d$ is a bounded set that contains zero, and $\varepsilon_i$'s are IID noises independent with $\vx_i$.
Moreover, we assume the following conditions. It is satisfied if $\varepsilon_i$ follows a sub-Gaussian or sub-exponential distribution.

\begin{assumption}\label{ass_data}
The noise terms satisfy $\E (\varepsilon_i^2 )  \le \tau^2$ and $\E (\max_{1 \leq i\leq n} |\varepsilon_i| ) \lesssim \tau\log n$.
\end{assumption}


\subsection{Deep neural network model class} \label{subsec_NNnotation}

Recall that our goal is to learn a regression function $\hat{f}_n: \vx \mapsto \hat{f}_n(\vx)$ for prediction. 
In practice, $\hat{f}_n$ is estimated from a pre-specified regression model class, such as linear regression and nonparametric regression based on series expansion with polynomials, splines, or wavelets bases. 
We consider the class of multi-layer feedforward neural networks (also illustrated in Figure~\ref{fig_diagram}), denoted by $\mathcal{F}_L$. 

\begin{figure}[tb]
\begin{center}
\centerline{\includegraphics[width=0.55\columnwidth]{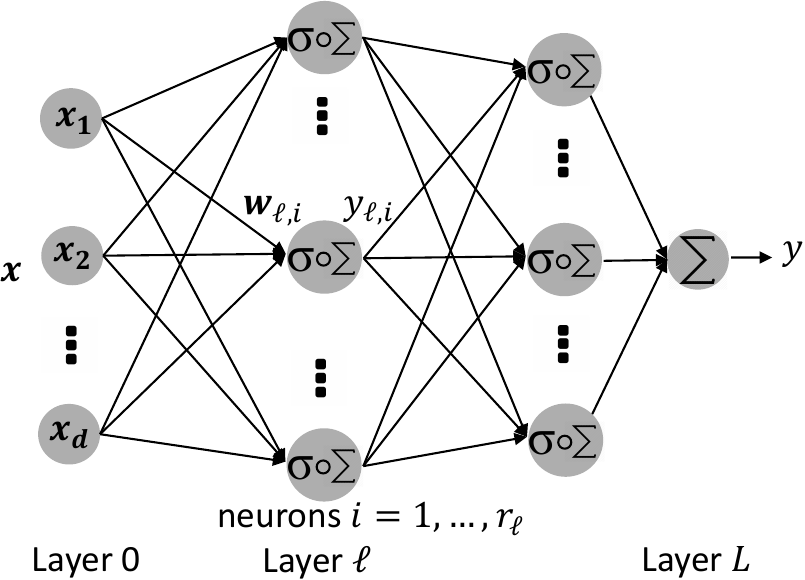}}
\caption{A graph showing multi-layer feedforward neural networks.}
\label{fig_diagram}
\end{center}
\vskip -0.3in
\end{figure}

In addition to the input layer, the neural network consists of $L$-layers, indexed by $l=1,\ldots,L$, each containing $r_l$ neurons/nodes.
Specially, the $L$-th layer is the the output layer, which outputs the regression result $y$ (so $r_L = 1$). Also, we may treat the input layer as layer $l=0$ (so $r_0 = d$). 
Let $\bm y_l \in \mathbb{R}^{r_l}$ denote the output of $l$-th layer, which is also the input of $l+1$-th layer, for $\ell=0,\ldots,L-1$.
In the $\ell$th layer ($\ell=1,\ldots,L-1$), the $i$th neuron takes a linear combination of the previous layer's output, with weights (or coefficients) $\bm w_{l, i} \in \mathbb{R}^{r_{l-1}}$, and passes it through a nonlinear activation function $\sigma(\cdot)$. 
In our analysis, $r_1,\ldots,r_{L-1}$ are allowed to be arbitrarily large, meaning that the network may have arbitrarily many hidden neurons.  

Suppose that the input data dimension is $\vx \in \mathbb{X} \subset \mathbb{R}^d$.
By our notation, we have $\bm y_0 = \vx$ (output of the $0$th layer), $\bm y_{l, i} = \sigma(\bm w_{l, i}^{\top}\bm y_{l-1})$ (output of the $\ell$th layer, $1 \le \ell \leq L-1$, and $y = \bm y_L = \bm w_L^{\top}\bm y_{L-1}$ (the final output).
Without loss of generality, we do not separately consider the bias term in each neural network layer, because it can be absorbed into the neural weight. Specifically, we note that $\bm y_{l, i} = \sigma(\bm w_{l, i}^{\top}\bm y_{l-1} + b_{l, i})$ can be rewritten as $\bm y_{l, i} = \sigma(\tilde{w}_{l, i}^{\top}\tilde{y}_{l-1})$, where $\tilde{w}_{l, i} = [\bm w_{l, i}, \frac{b_{l, i}}{\sigma(0)}]$ and $\bm \tilde{y}_{l-1} = [\bm y_{l-1}, \sigma(0)]$, Thus, we can simply realize $\sigma(0)$ by adding all-zero weights to the corresponding neuron, and the same proof will apply. 
With a slight abuse of notation, we used $\bm y_j$ for the output of the layer $j$, while $y_i$ for the $i$th regression response. 
We will use $w_{l, i, j}$ to denote the $j$th entry of the vector $\bm w_{l, i}$.

\subsection{Activation function} \label{subsec_activation}

We consider the following class of nonlinear activation function $\sigma(\cdot)$ for the technical analysis.
A specific case is the popular activation function $\sigma(x)=1/\{1+\exp(-x)\}$, also known as the logistic function.
We note that some other popular activation functions, e.g.,  the rectifier function, do not satisfy the assumption. 

\begin{assumption} \label{ass_activation}
	The activation function $\sigma(\cdot)$ is a bounded function on $\mathbb{R}$ satisfying $\sigma(z) \to 1$ as $z \to \infty$ and $\sigma(z) \to 0$ as $z \to -\infty$.
	Also, its derivatives satisfy $\sup_{z \in \mathbb{R}} |\sigma^{(k)}(z)| \le M$, for any positive integer $k$ and a constant $M$ (may depend on $k$).
\end{assumption}



\subsection{Variation-constrained networks} \label{subsec_variation}

Recall that $w_{l, i}$ denotes the neural weights for the $i$th neuron at the $\ell$th layer.
For any positive integer $L$ and real value $v$, let $\mathcal{F}_L(v)$ denote the set of all the $L$-layer feedforward neural networks (defined in Subsection~\ref{subsec_NNnotation}) satisfying $\|\bm w_{l, i}\|_1 \le v$ for all $i,\ell$. By the above notation, $\mathcal{F}_L$ may be written as $\mathcal{F}_L(\infty)$.
The existing result that any continuous function can be approximated arbitrarily well by a two-layer perceptron with sigmoid activation functions~\cite{cybenko1989approximations} indicates that the closure of $\{f: f \in \mathcal{F}_2(v), v \geq 0 \}$ contains all the continuous functions. It is thus a quite expressive model class. Though there does not seem to exist a theory on the size for $\mathcal{F}_L(v)$ with a general $L$, it is conceivable that its expressive power is not smaller than $\mathcal{F}_2(v')$ for some $v'$ that depends on $v$. As a result, any smooth function can be expressed by a neural network with sufficiently many layers and neurons.

From now on, we suppose that the number of layers $L$ is fixed. For any function $f_*$, we define its variation to~be
\begin{align}
V(f_*) := \inf\{v: f_* \in \mathcal{F}_L(v)\} + 1 \nonumber, 
\end{align}
Intuitively, $V(f_*)$ characterize the difficulty of learning $f_*$. The value of $V(f_*)$ is small for a smooth $f_*$ (e.g., a sigmoid function), and large for a bumpy $f_*$.
We note that $V(f_*)$ can be infinity even if $f_*$ is a continuous function, since $f_*$ may not be in the interior of the closure of $\{f: f \in \mathcal{F}_L(v), v \geq 0 \}$.
In the sequel, we mainly consider the case $V(f_*) < +\infty$,
and we will discuss the case $V(f_*) = +\infty$.

\subsection{Training and evaluation} \label{subsec_loss}

We learn the neural network from $n$ observations by solving the empirical risk minimizing problem
\begin{align}
	\hat{f}_n = \argmin_{f \in \mathcal{F}_L(V)} \frac{1}{n}  \sum_{i = 1}^n \l(f(\vx_i), y_i)  \label{eq_loss}
\end{align} 
for some loss function $\l(\cdot)$, where  $V$ is a chosen constant that constrains the variation of $\hat{f}_n$.
In practice, the above optimization can be operated by the alternative problem
\begin{align}
	\hat{f}_n = \argmin_{f \in \mathcal{F}} \biggl\{ \frac{1}{n}  \sum_{i = 1}^n \l(f(\vx_i), y_i) + \Omega(f) \biggr\} ,\label{eq_loss2}
\end{align} 
where $\mathcal{F}$ is the unconstrained class of neural networks
and $\Omega(\cdot)$ is an appropriately chosen functional. We will revisit the choice of $\Omega$ in later sections.

In the problem (\ref{eq_loss}), the loss function $\l(\cdot)$ is usually pre-determined to be the square $\Ltwo$ loss defined by $\Ltwo(y',y)^2 = (y'-y)^2$.
Correspondingly, the predictive performance of $\hat{f}_n$ is often evaluated by 
\begin{align}
\E ( \l(\hat{f}_n(\tilde{\bm x}, \tilde{y}) )
&= \E ( \l(\hat{f}_n(\tilde{\bm x}, \tilde{y}) )^2 \nonumber\\
&= \E ( \hat{f}_n(\tilde{\bm x}) - f_*(\tilde{\bm x}))^2 + \sigma^2  ,\label{eq_100}
\end{align}
where $(\tilde{\bm x}, \tilde{y})$ denotes a future observation independent with the data used to train $\hat{f}_n$. 
Removing the constant $\sigma^2$ and taking a square root for (\ref{eq_100}), we obtain the $\Ltwo$ risk introduced in (\ref{eq_l2risk}).
Another loss, the $\Lone$ loss defined by $\l(y',y) = |y'-y|$, is also used in practice. The training objective and the evaluation metric can be similarly defined.  
We will focus on the $\Ltwo$ loss in Section~\ref{sec_main}, and discuss the counterpart results for the $\Lone$ loss in Section~\ref{sec_L1}.

\section{Main Result} \label{sec_main}

%

Suppose that we use square $\Ltwo$ loss to train, namely
\begin{align}
	\hat{f}_n = \argmin_{f \in \mathcal{F}_L(V)} \frac{1}{n}  \sum_{i = 1}^n (f(\vx_i)- y_i)^2  . \label{eq_lossl2}
\end{align}

\begin{theorem} \label{thm:deep-nets}
Under Assumptions~\ref{ass_data} and \ref{ass_activation}, 
for an arbitrarily small constant $\delta > 0$ and any constant $V$ with $V \ge \max\{V(f_*),1\}$, the $\hat{f}_n$ in (\ref{eq_lossl2}) satisfies 
\begin{equation}
\R_2(\hat{f}_n)^2 = \E |\hat{f}_n(\bm x) - f_*(\bm x)|^2 \lesssim V^{2+\delta+(L-1)d} \ n^{-1 + 2\delta}. \nonumber
\end{equation}
\end{theorem}

\begin{remark}[Rate as a function of $n$]
The above result shows that the $\Ltwo$ risk $\R_2(\hat{f}_n)$ for the variation-constrained deep neural network class can achieve nearly parametric convergence rates  $n^{-1/2+\delta}$. 
The result implies that
\textit{for multi-layer neural networks with fixed variation and input dimension, the statistical risk decays with the sample size in a way similar to parametric models.} The implication is appealing in practice when we have a massive amount of data for learning an underlying function that is not so bumpy.

Though our convergence rate is tight only for fixed $d$ and $L$, its dependence $n$ is highly nontrivial to derive, even for a small input dimension such as $d = 2$ or $d=3$.
In terms of the dependence on $n$, our rate of convergence is much tighter than the existing results of $O(n^{-1/4})$ in various settings (see Section~\ref{sec_intro}).
In fact, we will show in Proposition~\ref{prop:minimax} that the $n^{-1/2}$ is the best rate one could expect for neural networks.
We will relax the dependence on $d$ in Section~\ref{sec_L1}.

Another interesting implication of Theorem~\ref{thm:deep-nets} is that the predictive performance of a class of neural networks does not necessarily depend on their `nominal complexity' as described by the number of neurons and layers. Instead, it only depends on the variation of the underlying function and $n$. The result indicates that deep neural networks may not suffer from overfitting when the network model is excessively complex, a striking observation made in many application studies. 

Figure~\ref{fig_result} shows an experiment result that initially motivated our study. 
We generated $n$ samples from $y=\beta_1x_1 + \cdots+\beta_5 x_5+\varepsilon$, where $x_1,\ldots,x_5,\varepsilon$ followed the IID standard Gaussian, and $n=2^5,2^6,\ldots, 2^{11}$. 
We used a $[5, 50, 10, 1]$-neural network, which consists of an input layer, two hidden layers with $50$ and $10$ neurons, and an output layer, and implemented the training using \textit{Pytorch}~\cite{ketkar2017introduction}. We recorded the squared $\Ltwo$ risk (also named the mean squared error) from 50 independent replications, using an independent test data with a size of $10^4$. 
To maintain a fixed variation, we approximately control the total $\ell_1$ norm of the parameters (around $250$) by applying $\ell_1$-regularizations to the neural weights during the training.  
\end{remark}

\begin{figure}[tb]
\begin{center}
\centerline{\includegraphics[width=0.55\columnwidth]{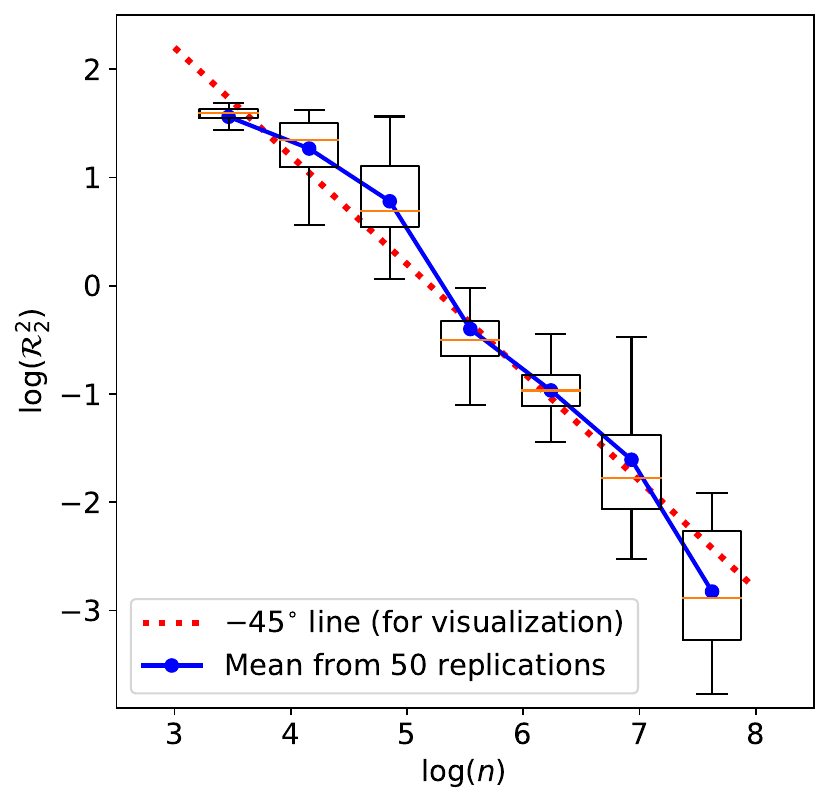}}
\caption{Experimental results showing the squared $\Ltwo$ risk of a class of three-layer feedforward neural networks with a fixed variation. The boxplots and mean values were calculated from 50 independent replications.}
\label{fig_result}
\end{center}
\vskip -0.3in
\end{figure}

\begin{remark}[No `free lunch']

The statistical risk of deep neural networks, as a nonparametric regression method, depends on many factors, including the sample size, the input dimension, and the smoothness of the underlying function.
Though Theorem~\ref{thm:deep-nets} gives a rate of convergence much faster than the typical nonparametric rates, a result not previously recognized even for a fixed $d$, it is derived by assuming a bounded variation $V$.
Without variation constraint, the rate could still be much slower than the parametric rate. Thus, Theorem~\ref{thm:deep-nets} does not conflict with the usual impression that neural networks may sacrifice the convergence rate for expressive power.

Specifically, recall that the set of continuous functions can be stratified by sets in the form of $\mathcal{F}_L(V)$ ($V\geq 0$), and that $V \ge V(f_*)$ implies that $f_* \in \mathcal{F}_L(V)$.
Theorem~\ref{thm:deep-nets} implies that only if $f_*$ can be expressed by a $V$-variation neural network, the predictive performance of the learned $\hat{f}_n$ will improve in a similar manner as a parametric model as $n$ becomes large.
The result does not hold if $f_*$ has a large (or infinite) variation.
This will be further discussed in the next remark.
\end{remark}

\begin{remark}[A new perspective of bias-variance tradeoffs in nonparametric regression]

The tight statistical risk rate (regarding $n$) is at the cost of restricting the value of variation. 
Recall that the variation of the underlying data-generating function means the smallest norm of neural weights needed to construct it. 
A general function $f_*$ may require a large or even infinite variation.
Though a large variation supports more regression functions with desirable approximation errors, it will enlarge the estimation variance and thus degrade the overall predictive performance. 
This naturally motivates the bias-variance tradeoff from a variation perspective. 

We define $T=V^{L}$, which represents the total variation of the network.
We introduce
\begin{align}
	\beta(f_*) = \lim\inf_{V \to \infty} \frac{-\log \min_{f \in \mathcal{F}_L(V)} \E\|f - f_*\|^2}{\log T} , \nonumber
\end{align}
which is well defined even for $V(f_*) = +\infty$.
An interpretation of $\beta(f_*)$ is that $f_*$ can be approximated by a neural network in $\mathcal{F}_L(V)$ with mean squared error no larger than $T^{-2\beta}$, for any constant $\beta < \beta(f_*)$.
Then, Theorem~\ref{thm:deep-nets} implies that
\begin{equation}
\E |\hat{f}_n(\bm x) - f_*(\bm x)|^2 \lesssim V^{2+\delta}(T/V)^{d}n^{-1 + 2\delta} + T^{-2\beta}
\end{equation}
for any positive constant $\delta$.
It can be easily verified that if $\log n \lesssim V^{d-2}$, we have
\begin{align}
\R_2(\hat{f}_n) = \sqrt{\E |\hat{f}_n(\bm x) - f_*(\bm x)|^2} = O(n^{-\frac{\beta}{d+2\beta}}),
\label{eq_101}
\end{align}
for any positive constant $\beta < \beta(f_*)$.
The risk bound in (\ref{eq_101}) is a consequence of bias-variance tradeoffs regarding the choice of $T$ (and thus $V$).
We note that (\ref{eq_101}) resembles the classical result of $O(n^{-\alpha/(d+2\alpha)})$ for the $\Ltwo$ risk of smoothness function classes (Subsection~\ref{subsec_nonpar}). 
A comprehensive analysis of the above tradeoffs for different neural structures is beyond the scope of this paper and is left as future work.

%

\end{remark}


\begin{proof}
The proof relies on an $\epsilon$-net argument.
The $L_{\infty}(\mathbb{X})$ metric of $\mathcal{F}(V)$ refers to:
\begin{align}
\|f_1-f_2\|_{L_{\infty}(\mathbb{X})} = \max_{\bm x \in \mathbb{X}}|f_1(\bm x) - f_2(\bm x)|.
\end{align}
We use $\mathcal{N}_{\epsilon}$ to denote an $\epsilon$-covering of $\mathcal{F}(V)$ under the above $L_{\infty}(\mathbb{X})$ metric. 
The proof will be based on the following two key observations.
\begin{itemize}
\item First, the $\epsilon$-covering is dense enough.
\begin{lemma}[Proved in the Appendix \ref{sec:proof-lem-eps-net-error}] \label{lem:eps-net-error}
If 
$ 
\|f-\hat{f}_n\|_{L_{\infty}(\mathbb{X})} \le \epsilon, 
$ 
then
\begin{subequations}
\begin{align}
\E (\hat{f}_n(\vx) - f_*(\vx))^2 - \E (f(\vx) - f_*(\vx))^2 \le&~4V\epsilon,\nonumber \\
\frac1n \sum_{i = 1}^n \biggl\{(f(\vx_i) - f_*(\vx_i) - \varepsilon_i)^2 - (\hat{f}_n(\vx_i) - f_*(\vx_i) - \varepsilon_i)^2\biggr\} \le&~4(V+\max_i |\varepsilon_i|)\epsilon.\nonumber
\end{align}
\end{subequations}
\end{lemma}
\item Second, there exists a finite $\epsilon$-covering.
\begin{lemma}[Proved in the Appendix \ref{sec:proof-lem-eps-net-number}] \label{lem:eps-net-number}
The covering number of $\mathcal{F}(V)$ satisfies
\begin{equation} 
\log N(\epsilon) \le \nu V^{d(L-1)+\delta}\epsilon^{-\delta}, \label{eq:metric_L}
\end{equation}
where $\nu$ is a constant that only depends on $d$ and $\delta$. 
\end{lemma}
\end{itemize}

With the above two lemmas, we first make the following decomposition.
\begin{align}
\E (\hat{f}_n(\bm x) - f_*(\bm x))^2 =&~\E (\hat{f}_n(\bm x) - f_*(\bm x))^2 - \E (f(\bm x) - f_*(\bm x))^2 \nonumber\\
&+ \E (f(\bm x) - f_*(\bm x))^2 -  \frac{1}{n} \sum_{i = 1}^n [(f(\bm x_i) - f_*(\bm x_i) - \varepsilon_i)^2 - \varepsilon_i^2] \nonumber\\
&+  \frac{1}{n} \sum_{i = 1}^n (f(\bm x_i) - f_*(\bm x_i) - \varepsilon_i)^2 -  \frac{1}{n} \sum_{i = 1}^n (\hat{f}_n(\bm x_i) - f_*(\bm x_i) - \varepsilon_i)^2 \nonumber\\
&+\frac{1}{n} \sum_{i = 1}^n [(\hat{f}_n(\bm x_i) - f_*(\bm x_i) - \varepsilon_i)^2 - \varepsilon_i^2],\label{eq_103}
\end{align}
where $f \in \mathcal{N}_{\epsilon}$ satisfies
\begin{align}
\|f-\hat{f}_n\|_{L_2(\mu)} \le \epsilon.
\end{align}

Next, we bound the four terms in (\ref{eq_103}) separately.
\begin{itemize}
\item According to Lemma \ref{lem:eps-net-error}, the first term can be bounded by $4V\epsilon$.

\item The second term is bounded by applying Bernstein's inequality and the union bound on $\mathcal{N}_{\epsilon}$, which is formalized by the following lemma.

\begin{lemma}[Proved in the Appendix \ref{sec:proof-lem-gen-error}] \label{lem:gen-error}
For any given $\rho>0$
\begin{align} 
&\E (f(\vx) - f_*(\vx))^2 - \frac{1}{n} \sum_{i = 1}^n [(f(\bm x_i) - f_*(\bm x_i) - \varepsilon_i)^2 - \varepsilon_i^2] \nonumber\\
&\lesssim \frac{V(V+\max_i |\varepsilon_i|)}{n}\log \frac{N(\epsilon)}{\rho}. \label{eq:risk_L2_net}
\end{align}
holds uniformly for all $f \in \mathcal{N}_{\epsilon}$ with probability at least $1 - \rho$.
\end{lemma}

\item Lemma \ref{lem:eps-net-error} indicates that the third term can be bounded by $4(V+\max_i |\varepsilon_i|)\epsilon$.

\item As for the fourth term, according to the definition of $\hat{f}_n$, we have 
$$\frac1n \sum_{i = 1}^n (\hat{f}_n(\vx_i) - f_*(\vx_i) - \varepsilon_i)^2 \le \frac1n \sum_{i = 1}^n \varepsilon_i^2,$$
which implies that
\begin{align}
\frac{1}{n} \sum_{i = 1}^n [(\hat{f}_n(\bm x_i) - f_*(\bm x_i) - \varepsilon_i)^2 - \varepsilon_i^2] \le 0. \nonumber
\end{align}

\end{itemize}

Therefore, we have 
\begin{align}
\E (\hat{f}_n(\bm x) - f_*(\bm x))^2 
&\lesssim V\epsilon + \frac{V(V+\max_i |\varepsilon_i|)}{n}\log \frac{N(\epsilon)}{\rho} + (V+\max_i |\varepsilon_i|)\epsilon \label{eq_105}
\end{align}
uniformly for all $f \in \mathcal{N}_{\epsilon}$ with probability at least $1 - \rho$.

Let $E_n$ denote the event that (\ref{eq_105}) occurs for all $f \in \mathcal{N}_{\epsilon}$. Then $\P(E_n^c) \leq \rho$. 
Recall that $\E \max_i |\varepsilon_i| \lesssim \tau\log n$ (Assumption~\ref{ass_data}).
Taking $\rho = 1/n$ and invoking (\ref{eq_105}),  we have
and for all $f \in \mathcal{N}_{\epsilon}$,
\begin{align*} 
\E (\hat{f}_n(\bm x) - f_*(\bm x))^2 
&= 
\E\biggl\{ (\hat{f}_n(\bm x) - f_*(\bm x))^2  \i_{E_n} + (\hat{f}_n(\bm x) - f_*(\bm x))^2 \i_{E^c_n} \biggr\}  \\
& \lesssim \frac{V(V+\tau\log n)}{n}\log \frac{N(\epsilon)}{\rho} + (V+\tau\log n)\epsilon
+ V^2 \P(E^c_n). \\
&\lesssim \frac{V(V+\tau\log n)}{n}\log (nN(\epsilon)) + (V+\tau\log n)\epsilon.
\end{align*}
 
By using the upper bounds for $N(\epsilon)$ in \eqref{eq:metric_L} and letting $\epsilon = 1/n$, we obtain
\begin{equation}
\R_2(\hat{f}_n)^2 = \E |\hat{f}_n(\bm x) - f_*(\bm x)|^2 \lesssim V^{1+\delta+(L-1)d}(V+\tau\log n) n^{-1 + \delta}. \label{eq_107}
\end{equation}
The right-hand side in (\ref{eq_107}) is further upper bounded by $V^{2+\delta+(L-1)d} \ n^{-1 + 2\delta}$, which concludes the proof.

\end{proof}


At the end of this section, we show that the minimax risk lower bounds for the neural network class $\mathcal{F}_2(V)$ is at the order of $n^{-1/2}$.
The result implies that the upper bound in Theorem~\ref{thm:deep-nets} is rather tight. 
We need the following additional assumption.
\begin{assumption} \label{ass_activation2}
	There exists a constant $c > 0$ and a bounded subset $\mathcal{S} \subset \mathbb{R}$ such that $\mathbb{P}(Z \in \mathcal{S}) > c$ and $\inf_{z \in \mathcal{S}} \sigma'(z) > c$ for $Z \sim \mathcal{N}(0, 1)$. 
\end{assumption}

\begin{proposition}[Minimax risk lower bound]
\label{prop:minimax}
Suppose that Assumptions~\ref{ass_activation} and \ref{ass_activation2} hold,
and $\vx_1, \vx_2, \ldots, \vx_n \overset{iid}{\sim} \mathcal{N}(0, \vI_d)$ (with the probability measure denoted by $\mu$).
Then, for $q = 1, 2$,
\begin{equation}
\inf_{\hat{f}_n}\sup_{f \in \mathcal{F}_2(V)} \|\hat{f}_n - f\|_{L_q(\mu)} \gtrsim V\sqrt{\frac{d}{n}}.
\end{equation}
\end{proposition}

\begin{proof}
The proof is postponed to Appendix \ref{sec:proof-prop-minimax}.
\end{proof}

\section{Result under the $\Lone$ Risk and Improvement of the Dependence on $V$} \label{sec_L1}


Suppose that we use the $\Lone$ loss for training, namely
\begin{align}
	\hat{f}_n = \argmin_{f \in \mathcal{F}_L(V)} \frac{1}{n}  \sum_{i = 1}^n |f(\vx_i)-y_i|  . \label{eq_lossl1}
\end{align} 
Correspondingly, we evaluate the predictive performance of a learned regression function $f$ using the $\Lone$-statistical risk 
$$
\R_1(f) = \E |f(\vx)-y| - \E |\varepsilon|.
$$
It can be verified that $\R_1(f)$ is nonnegative for symmetric random variables $\varepsilon$.
We note that the $\Ltwo$ risk is equivalent to the squared $\Ltwo(\mu)$ metric plus some constant (see (\ref{eq_100}).
But minimizing the $\Lone$ risk does not necessarily minimize the $\Lone(\mu)$ metric. 
We are interested in the $\Lone$-based training and evaluation mainly because they are practically used in many learning problems, e.g., those for heterogeneous noises~\cite{Welsch1977}, ordinal data~\cite{pedregosa2017consistency}, and imaging data~\cite{zhao2016loss}.
From a practical point of view, using $\Lone$ loss for training is as easy and fast as using $\Ltwo$ in prevalent computational frameworks such as 
\textit{Tensorflow}~\cite{abadi2016tensorflow}, 
\textit{Pytorch}~\cite{ketkar2017introduction}, and
\textit{Keras}~\cite{gulli2017deep}.

From the theoretical perspective, interestingly, we found that the $\Lone$ loss provides  technical convenience for deriving tight risk bounds.
In particular, we will show that the bound in Theorem~\ref{thm:deep-nets} can be improved in terms of the dependence on the input dimension $d$.

\begin{theorem} \label{thm:l1out}
Under Assumption \ref{ass_activation}, 
for any constant $V$ with $V \ge V(f_*)$, the $\hat{f}_n$ in (\ref{eq_lossl1}) satisfies 
\begin{align}
\R_1(\hat{f}_n) \lesssim \frac{V^{L-1}\sqrt{d\log n}+\tau}{\sqrt{n}}. \label{eq:l1risk}
\end{align}
\end{theorem}

\begin{proof}
The proof is postponed to Appendix~\ref{sec:proof-thm-l1out}.
\end{proof}

\begin{remark}[Explicit regularization]
Theorem~\ref{thm:l1out} shows that variation-constrained neural networks, when trained and evaluated under the $\Lone$ loss, does not explicitly depend on the input dimension $d$.
Compared with the result in Theorem~\ref{thm:deep-nets}, the rate dependence on $n$ is similar, but the reliance on the variation $V$ is much relaxed. 

In practice, we can operate the following regularized optimization to reach the desirable statistical risk.
According to the method of Lagrange multipliers,  
the constrained optimization problem in (\ref{eq_lossl1}) can be formulated as 
\begin{equation} \label{eq:out_reg}	
\hat{f}_n = \argmin_{f \in \mathcal{F}} \frac1n\sum_{i = 1}^n|y_i - f(\vx_i)| + \lambda V(f)	^L
\end{equation}	
where $\mathcal{F}$ is the same neural network class without any constraint, and $\lambda$ is some appropriately chosen parameter.
Theorem \ref{thm:l1out} implies that for any $f \in \mathcal{F}$,
\begin{equation*} 
\R_{1}(f) \le \R_{1,n}(f) + O\biggl(\sqrt{\frac{\log n}{n}}\biggr) V(f)^L + O\biggl(\frac{\tau}{\sqrt{n}}\biggr),
\end{equation*}
which further implies that $\lambda$ can be chosen at the order of $ O(\sqrt{(\log n)/n})$ in \eqref{eq:out_reg}.

\end{remark}

\begin{remark}[Implication on neural network model selection]

There are two general ways of selecting a neural network model in practice. One is to consider a set of candidate architectures and choose the one with the best cross-validation performance.
Classical asymptotic analysis of the generalization error alludes that more neural weights (or neurons) tend to cause overfitting. For example, an information criterion-type derivation indicates that the generalization error typically grows linearly with the number of free parameters~\cite{DingOverview}. 
Nevertheless, recent research has shown that an overly-large network does not necessarily cause overfitting~\cite{zhang2016understanding}.
A possible reason is the failure of regularity conditions traditionally required for M-estimators. 

The other way, which is perhaps more prevalent in practice these days, is to train an extensive neural network with properly tuned regularization terms. 
A practical benefit of the second approach is its more straightforward hardware implementation and computation, as we do not need to implement and train multiple models separately. 
Theoretically, Theorem~\ref{thm:l1out} provides insight on when overfitting will not occur.

\end{remark}

%
%
%
%
%
%
%
%
%

\section{Conclusion} \label{sec_conclusion}

In this work, we showed tight statistical risk bounds for variation-constrained deep neural networks under both $\Lone$ and $\Ltwo$ loss functions.
Several related problems need further study.
First, a similar analysis may be emulated to study the performance of deep neural network-based classification models. 
Second, it would be interesting to study the $\Ltwo$ risk of a model that is trained from the $\Lone$ empirical risk, especially under non-IID or heavy-tail noises. 
The third problem is to study the relationship between the variation-based regularization and implicit regularization techniques (e.g., the early stopping and the dropout) practically operated in training deep neural networks.

\begin{appendix}

\section{Proof of Technical Lemmas}

In this Appendix, we prove the technical lemmas used in the proof of Theorem~\ref{thm:deep-nets}.

\subsection{Proof of Lemma \ref{lem:eps-net-error}}
\label{sec:proof-lem-eps-net-error}

Since 
\begin{align}
\|f-\hat{f}_n\|_{L_{\infty}(\mathbb{X})} \le \epsilon,
\end{align}
we have
\begin{align*}
&\E (\hat{f}_n(\vx) - f_*(\vx))^2 - \E (f(\vx) - f_*(\vx))^2 \\
=&~\E (\hat{f}_n(\vx) - f_*(\vx) + f(\vx) - f_*(\vx))(\hat{f}_n(\vx) - f(\vx)) \\
\le&~4V\E |\hat{f}_n(\vx) - f(\vx)| \\
\le&~4V\epsilon.
\end{align*}
and similarly,
\begin{align*}
&\frac1n \sum_{i = 1}^n \biggl\{(f(\vx_i) - f_*(\vx_i) - \varepsilon_i)^2 - (\hat{f}_n(\vx_i) - f_*(\vx_i) - \varepsilon_i)^2\biggr\} \\
=&~\frac1n \sum_{i = 1}^n \biggl\{(f(\vx_i) - f_*(\vx_i) - \varepsilon_i + \hat{f}_n(\vx_i) - f_*(\vx_i) - \varepsilon_i)(f(\vx_i) - \hat{f}_n(\vx_i))\biggr\} \\
=&~\frac1n \sum_{i = 1}^n (f(\vx_i) - f_*(\vx_i) + \hat{f}_n(\vx_i) - f_*(\vx_i))(f(\vx_i) - \hat{f}_n(\vx_i)) \\
&\quad - \frac1n \sum_{i = 1}^n \bigl\{ 2\varepsilon_i(f(\vx_i) - \hat{f}_n(\vx_i)) \bigr\} \\
\le&~(4V+2\max_i \varepsilon_i)\frac1n \sum_{i = 1}^n |f(\vx_i) - \hat{f}_n(\vx_i)| \\
\le&~4(V+\max_i \varepsilon_i)\epsilon.
\end{align*}

\subsection{Proof of Lemma \ref{lem:eps-net-number}}
\label{sec:proof-lem-eps-net-number}

We first show by the induction method that for any $1\le l<L$ and $1\le i \le r_l$,
\begin{equation} \label{eq:derivative}
\|\nabla^{k} \bm y_{l,i}\|_{\infty} \lesssim V^{lk}
\end{equation} 
for $k \le \alpha d$,
where we let $\alpha = 1/\delta$ and $\nabla$ is the differentiation operator on $\bm x$. 
First, according to Assumption~\ref{ass_activation}, we have for $l = 1$,
\begin{align*}
\|\nabla^{k} \bm y_{l,i}\|_{\infty} =&~\|\nabla^{k} \sigma(\bm w_{l, i}^{\top}x)\|_{\infty} \\
=&~\|\sigma^{(k)}(\bm w_{l, i}^{\top}x)\bm w_{l, i}^{\otimes k}\|_{\infty} \\
\lesssim&~\|\bm w_{l, i}^{\otimes k}\|_{\infty} \\
\le&~\|\bm w_{l, i}\|_{\infty}^k \le V^k.
\end{align*}
Assume that \eqref{eq:derivative} holds for $1, \ldots, l$.
Next, we will prove the claim also holds for $\bm y_{l+1,i} = \sigma(\bm w_{l+1, i}^{\top}\bm y_l)$ by showing that
\begin{equation} \label{eq:derivative_2}
\|\nabla^{k} \sigma^{(s)}(\bm w_{l+1, i}^{\top}\bm y_l)\|_{\infty} \lesssim V^{(l+1)k}
\end{equation} 
for any integers $1 \le i \le r_{l+1}$ and $1 \leq s \le \alpha d - k$,  also by using the induction method (on $k$) 
For $k = 1$, we have,
\begin{align*}
\|\nabla \sigma^{(s)}(\bm w_{l+1, i}^{\top}\bm y_l)\|_{\infty} =&~\|\sigma^{(s+1)}(\bm w_{l+1, i}^{\top}\bm y_l)\nabla (\bm w_{l+1, i}^{\top}\bm y_l)\|_{\infty} \\
\lesssim&~\|\nabla (\bm w_{l+1, i}^{\top}\bm y_l)\|_{\infty} \\
\le&~V\max_{1\le j\le r_l}\|\nabla \bm y_{l, j}\|_{\infty} \\
\lesssim&~V^{l+1}.
\end{align*}
Assume that \eqref{eq:derivative_2} holds for $1, \ldots, k-1$.
Then, for $s \le \alpha d - k$
\begin{align*}
\bigl\|\nabla^{k} \sigma^{(s)}(\bm w_{l+1, i}^{\top}\bm y_l)\bigr\|_{\infty} 
=&~\bigl\|\nabla^{k-1} \nabla \sigma^{(s)}(\bm w_{l+1, i}^{\top}\bm y_l)\bigr\|_{\infty} \\
=&~\bigl\|\nabla^{k-1} [\sigma^{(s+1)}(\bm w_{l+1, i}^{\top}\bm y_l)\nabla (\bm w_{l+1, i}^{\top}\bm y_l)]\bigr\|_{\infty} \\
=&~\biggl\|\nabla^{k-1} \sigma^{(s+1)}(\bm w_{l+1, i}^{\top}\bm y_l) \otimes \nabla (\bm w_{l+1, i}^{\top}\bm y_l) \\
	&\quad + \sigma^{(s+1)}(\bm w_{l+1, i}^{\top}\bm y_l)\nabla^{k} (\bm w_{l+1, i}^{\top}\bm y_l)\biggr\|_{\infty} \\
\lesssim&~\biggl\|\nabla^{k-1} \sigma^{(s+1)}(\bm w_{l+1, i}^{\top}\bm y_l) \otimes \nabla (\bm w_{l+1, i}^{\top}\bm y_l)\|_{\infty} + \|\nabla^{k} (\bm w_{l+1, i}^{\top}\bm y_l)\biggr\|_{\infty} \\
\le&~\biggl\|\nabla^{k-1} \sigma^{(s+1)}(\bm w_{l+1, i}^{\top}\bm y_l)\|_{\infty} \|\nabla (\bm w_{l+1, i}^{\top}\bm y_l)\|_{\infty} + \|\nabla^{k} (\bm w_{l+1, i}^{\top}\bm y_l)\biggr\|_{\infty} \\
\lesssim&~V^{(l+1)(k-1)}\cdot \biggl\|\sum_{j=1}^{r_{l}} w_{l+1, i, j} \bm y_{l,j} \biggr\|_{\infty} + V\cdot V^{lk} \\
\lesssim&~V^{(l+1)(k-1)}\cdot V\cdot V^{l} + V\cdot V^{lk} \lesssim V^{(l+1)k}.
\end{align*}
Therefore, we proved the bound in \eqref{eq:derivative_2}, which further implies the bound in \eqref{eq:derivative}.

Consequently, for any function $f \in \mathcal{F}_L(V)$ and $k \le \alpha d$,
\begin{align*}
\|\nabla^{k} f(\bm x) \|_{\infty}
= &\|\bm w_{L}^\T  \nabla^{k} \bm y_{L-1}\|_{\infty} \\
\le&~V\max_{1 \le i \le r_{L-1}} \|\nabla^{k} \bm y_{L-1, i}\|_{\infty} \\
\lesssim&~V\cdot V^{k(L-1)}.
\end{align*}
The above result implies $\mathcal{F}(V)$ belongs to the smooth functions class with order $\alpha d$. 


Then according to Lemma \ref{lem:entropy}, the covering number of $\mathcal{F}(V)$ satisfies
\begin{equation} 
\log N(\epsilon) \le \nu (V^{\alpha d(L-1) + 1})^{\frac{1}{\alpha}}\epsilon^{-\frac{1}{\alpha}} = \nu V^{d(L-1)+\frac{1}{\alpha}}\epsilon^{-\frac{1}{\alpha}}, \nonumber
\end{equation}
where $\nu$ is a constant that only depends on $d, \alpha$. We conclude the proof by invoking $\alpha = 1/\delta$.

\subsection{Proof of Lemma \ref{lem:gen-error}}
\label{sec:proof-lem-gen-error}

For any fixed $f \in \mathcal{N}_{\epsilon}$, we let
$\xi_i(f) = (f(\vx_i) - f_*(\vx_i) - \varepsilon_i)^2 - \varepsilon_i^2$, or $\xi_i$ for notational simplicity.
It can be verified that 
$$\E (\xi_i )= \E (f(\vx) - f_*(\vx))^2,$$ and 
\begin{align*}
\var(\xi_i) =&~\var\{(f(\vx_i) - f_*(\vx_i))^2 - 2\varepsilon_i(f(\vx_i) - f_*(\vx_i))\} \\
\le&~\E \{(f(\vx_i) - f_*(\vx_i))^2 - 2\varepsilon_i(f(\vx_i) - f_*(\vx_i))\}^2 \\
=&~\E (f(\vx_i) - f_*(\vx_i))^4 + 4\E \varepsilon_i^2(f(\vx_i) - f_*(\vx_i))^2 \\
\le&~4(V^2+\tau^2)\E (f(\vx) - f_*(\vx))^2.  
\end{align*}
In the last inequality, we have used the inequality that
\begin{align}
f(\bm x)  
 &= \bm w_{L}^{\top}\bm y_{L-1}
 = \bm w_{L}^{\top} \, [ \sigma(\bm w_{L-1, i}^{\top}\bm y_{L-2}), \ldots, \sigma(\bm w_{L-1, i}^{\top}\bm y_{L-2})]^\T \nonumber \\
 &\leq \|\bm w_{L}^{\top}\|_{1} \leq V , \quad \forall \bm x \in \mathbb{X}, \, \forall f\in \mathcal{F}_L(V)   . \label{eq_boundf}
\end{align}
According to the Bernstein's Inequality (Lemma~\ref{lemma_Bernstein} in the Appendix) and (\ref{eq_boundf}), 
 with probability at least $1 - \frac{\rho}{N(\epsilon)}$ we have
\begin{align}
&\E (f(\vx) - f_*(\vx))^2 - \frac1n \sum_{i = 1}^n \xi_i \nonumber \\
&\leq \frac{\max_i |\xi_i - \E (f(\vx) - f_*(\vx))^2 |}{n} \log \frac{N(\epsilon)}{\rho}
	+ \sqrt{\frac{\var(\zeta_1)}{n}\log \frac{N(\epsilon)}{\rho}} \nonumber \\
&\lesssim \frac{V(V+\max_i |\varepsilon_i|)}{n}\log \frac{N(\epsilon)}{\rho} + \sqrt{\frac{(V^2+\tau^2)\E (f(\vx) - f_*(\vx))^2}{n}\log \frac{N(\epsilon)}{\rho}}.  \label{eq_104}
\end{align}
It can be verified that the above (\ref{eq_104}) implies that 
\begin{align} 
\E (f(\vx) - f_*(\vx))^2 - \frac1n \sum_{i = 1}^n \xi_i \lesssim \frac{V(V+\max_i |\varepsilon_i|)}{n}\log \frac{N(\epsilon)}{\rho}. 
\label{eq_event}
\end{align}

By the union bound over $f \in \mathcal{N}_{\epsilon}$, (\ref{eq_event}) holds uniformly for all $f \in \mathcal{N}_{\epsilon}$ with probability at least $1 - \rho$.

\section{Auxiliary Lemmas}

\begin{lemma}[Bernstein's Inequality] \label{lemma_Bernstein}
Assume $X_1, X_2, \ldots, X_n$ are independent random variables,
then with probability at least $1-\rho$,
\begin{equation}
\biggl|\sum_{i=1}^n (X_i - \E X_i)\biggr| \lesssim \max_{1\leq i \leq n}|X_i - \E X_i|\log \frac{1}{\rho} + \sqrt{\sum_{i=1}^n \var(X_i)\log \frac{1}{\rho}}. \nonumber 
\end{equation}
\end{lemma}

\begin{lemma}[Metric Entropy Bound] \label{lem:entropy}
Let $\mathcal{F}$ be the class of all functions $f : [0, 1]^d \to \mathbb{R}$, whose partial derivatives up to order $\alpha d$ (which is supposed to be a positive integer) exist and are uniformly bounded by a constant $M$.
Define the $L_{\infty}$ metric as $\|f_1-f_2\|_{L_{\infty}} = \max_{\bm x \in [0, 1]^d}|f_1(\bm x) - f_2(\bm x)|$, and $N(\varepsilon)$ is the minimum covering number under such a metric.  
Then,
\begin{equation}
\log N(\varepsilon) \le \nu M^{\frac{1}{\alpha}} \varepsilon^{-\frac{1}{\alpha}}, \nonumber
\end{equation}
where $\nu$ depends on $\alpha, d, M$ only.
\end{lemma}

\begin{proof}
For two real-valued functions $l$ and $u$, the bracket $[l, u]$ is the set of all functions $f$ satisfying $l \le f \le u$.
An $\varepsilon$-bracket in $L_2(\mu)$-space is a bracket $[l, u]$ with $\|u - l\|_{L_2(\mu)} \le \varepsilon$.
The bracketing number $N_{[]}(\varepsilon)$ is the smallest number of $\varepsilon$-brackets needed to cover the function class.
Since the $\varepsilon$-bracket $[l, u]$ is contained in the ball with a radius of $\frac{\varepsilon}{2}$ centered on $\frac{l+u}{2}$ in $L_2(\mu)$, an $\varepsilon$-bracket covering is an $\varepsilon$ covering,
and thus $N_{[]}(\varepsilon) \ge N(\varepsilon)$. 
Then, this lemma directly follows from \cite[Theorem 2.1]{van1994bracketing}.
\end{proof}

The following result is a version of Talagrand's contraction
lemma. 

\begin{lemma}[Contraction Lemma] \label{lem:contract}
Suppose that a function $g$ is $L$-Lipschitz and $g(0) = 0$.
Suppose that $\xi_1, \ldots, \xi_n $ are IID symmetric Bernoulli random variables taking values from $\{1,-1\}$.
Then, for any function class $\mathcal{F}$ mapping from $\mathbb{X}$ to $\mathbb{R}$, and any set $\{\vx_1, \vx_2, \ldots, \vx_n\}$, we have
\begin{equation}
\E \sup_{f \in \mathcal{F}}\biggl|\frac1n \sum_{i = 1}^n \xi_i g(f(\vx_i))\biggr| \le 2L \E \sup_{f \in \mathcal{F}}\biggl|\frac1n \sum_{i = 1}^n \xi_i f(\vx_i)\biggr|. \nonumber
\end{equation}
\end{lemma}


\begin{lemma}[Minimax lower bound]
\label{lem:minmax}
Suppose $d$ is a metric on $\mathcal{F}$.
Then we have
\begin{align}
\inf_{\hat{f}} \sup_{f \in \mathcal{F}} \E d^2(f, \hat{f}) \gtrsim \varepsilon^2,
\end{align}
where $\varepsilon$ satisfies
\begin{align}
M_d(\varepsilon) = 4V_k(\varepsilon) + 2\log 2.
\end{align}
\end{lemma}

\begin{proof}
The result directly follows from~\cite[Theorem 1]{yang1999information}.
\end{proof}

\section{Proof of Proposition~\robustrefThmMinimax}
\label{sec:proof-prop-minimax}

We define a subclass of $\mathcal{F}_2(V)$ by
\begin{equation}
\mathcal{F}_0 = \biggl\{f : \mathbb{R}^d \to \mathbb{R} \Bl f(\bm x) = V\sigma(\vw^{\top}\bm x), \|\vw\|_2 = 1\biggr\}.
\end{equation}
It can be verified that
\begin{align*}
\E |\sigma(\vw_1^{\top}\vx) - \sigma(\vw_2^{\top}\vx)| 
&\ge \E\biggl\{ \inf_{u} \sigma'(u) \cdot |\vw_1^{\top}\vx - \vw_2^{\top}\vx| \cdot \mathbb{I}(\vw_1^{\top}\vx, \vw_2^{\top}\vx \in \mathcal{S}) \biggr\}\\
&\gtrsim \|\vw_1 - \vw_2\|_2.
\end{align*}
Let $M_1(\epsilon)$ denote the packing $\epsilon$-entropy  
of $\mathcal{F}_0$ with the $\Lone$ distance.
Then, $M_1(\epsilon)$ is greater than the packing $\epsilon$-entropy with the $\Ltwo$ distance, written as
$
M_1(\epsilon) \ge M_2(\epsilon) \gtrsim d.
$
Let $V_k(\epsilon)$ denote the covering $\epsilon$-entropy of $\mathcal{F}_0$ with the square root Kullback-Leibler divergence,
then according to its relation with the $\Ltwo$ distance shown in~\cite{yang1999information},
we have
\begin{align}
V_k(\epsilon) \le M_2(\sqrt{2}\epsilon) \lesssim d\log \frac{1}{\epsilon}. \nonumber
\end{align}
Hence, according to Lemma~\ref{lem:minmax}, for $q = 1, 2$,
\begin{equation}
\inf_{\hat{f}_n}\sup_{f \in \mathcal{F}_2(V)} \|\hat{f}_n - f\|_{L_q(\mu)} \ge \inf_{\hat{f}_n}\sup_{f \in \mathcal{F}_0} \|\hat{f}_n - f\|_{L_p(\mu)} \gtrsim V\sqrt{\frac{d}{n}}. 
\end{equation}
This concludes the proof.

\section{Proof of Theorem~\robustrefThmLout}
\label{sec:proof-thm-l1out}

We define the empirical risk 
\begin{equation}
\R_{1,n}(f) = \E \biggl( \frac1n \sum_{i = 1}^n |f_*(\vx_i) + \varepsilon_i - f(\vx_i)| \biggr) - \E |\varepsilon|.
\end{equation}
Since $\hat{f}_n$ minimizes $n^{-1} \sum_{i = 1}^n |f_*(\vx_i) + \varepsilon_i - f(\vx_i)|$ in $\mathcal{F}_L(V)$,
we have
\begin{equation}
\R_{1}(\hat{f}_n) \le \R_{1}(\hat{f}_n) - \{\R_{1,n}(\hat{f}_n) - \R_{1,n}(f_*)\} = \R_{1}(\hat{f}_n) - \R_{1,n}(\hat{f}_n). \label{new11}
\end{equation}
In the following, we will analyze the term $\R_{1}(\hat{f}_n) - \R_{1,n}(\hat{f}_n)$ in (\ref{new11}). 

Let $\vz_i$'s denote IID copies of $\vx_i$'s.
\begin{align*}
\R_{1}(\hat{f}_n) - \R_{1,n}(\hat{f}_n) =&~\E \frac1n \sum_{i = 1}^n \biggl\{|\hat{f}_n(\vz_i) - f_*(\vz_i) - \varepsilon_i| - |\hat{f}_n(\vx_i) - f_*(\vx_i) - \varepsilon_i|\biggr\} \\
\le&~\E \sup_{f \in \mathcal{F}_L(V)}\frac1n \sum_{i = 1}^n \biggl\{|f(\vz_i) - f_*(\vz_i) - \varepsilon_i| - |f(\vx_i) - f_*(\vx_i) - \varepsilon_i|\biggr\}  \\
\le&~2\E \sup_{f \in \mathcal{F}_L(V)}\frac1n \sum_{i = 1}^n \xi_i |f(\vz_i) - f_*(\vz_i) - \varepsilon_i|,  
\end{align*}
where $\xi_1, \ldots, \xi_n$ are IID symmetric Bernoulli random variables that are independent with $\vz_i$'s and take values from $\{1,-1\}$. 
According to Lemma~\ref{lem:contract}, 
since $g(x) = |x|$ is $1$-Lipschitz and $g(0) = 0$,
we have
\begin{align*}
\E \sup_{f \in \mathcal{F}_L(V)}\frac1n \sum_{i = 1}^n \xi_i |f(\vz_i) - f_*(\vz_i) - \varepsilon_i| 
\le&~2\E \sup_{f \in \mathcal{F}_L(V)}|\frac1n \sum_{i = 1}^n \xi_i (f(\vz_i) - f_*(\vz_i) - \varepsilon_i)| \\
\le&~2\E \sup_{f \in \mathcal{F}_L(V)}\biggl|\frac1n \sum_{i = 1}^n \xi_i f(\vz_i)\biggr| + 2\sqrt{\frac{\E y^2}{n}} \\
\lesssim&~\E \sup_{f \in \mathcal{F}_L(V)}\biggl|\frac1n \sum_{i = 1}^n \xi_i f(\vz_i)\biggr| + \frac{V+\tau}{\sqrt{n}}.
\end{align*}

To conclude the proof, it remains to show that 
$$\E \sup_{f \in \mathcal{F}}\biggl|\frac1n \sum_{i = 1}^n \xi_i f(\vz_i)\biggr| \lesssim V^{L-1}\sqrt{\frac{d\log n}{n}}.$$
Let $W(f)=\{\bm w_1, \ldots, \bm w_L\}$ denote the neural weights of $f$ at all the layers.
It can be verified that
\begin{align}
\E \sup_{f \in \mathcal{F}_L(V)} \biggl|\frac1n \sum_{i = 1}^n \xi_i f(\vz_i)\biggr| 
=& \E \sup_{W(f): \, f \in \mathcal{F}_L(V)} \biggl|\langle \vw_L, \frac1n \sum_{i = 1}^n \xi_i \bm y_{L-1}(\vz_i) \rangle \biggr| \nonumber \\
\le& V \E \sup_{W(f): \, f  \in \mathcal{F}_L(V)}\biggl\|\frac1n \sum_{i = 1}^n \xi_i \bm y_{L-1}(\vz_i) \biggr\|_{\infty} \nonumber \\
\le& V \E \sup_{W(f): \, f  \in \mathcal{F}_L(V)} \biggl|\frac1n \sum_{i = 1}^n \xi_i \sigma(\bm w_{L-1, j}^{\top}\bm y_{L-2}(\vz_i)) \biggr| \nonumber \\
\lesssim& V \E \sup_{W(f): \, f  \in \mathcal{F}_L(V)} \biggl|\frac1n \sum_{i = 1}^n \xi_i \bm w_{L-1, j}^{\top}\bm y_{L-2}(\vz_i) \biggr| \nonumber \\
\lesssim& \ldots \nonumber\\
\lesssim& V^{L-1}\E \sup_{\vw_1 \in \mathbb{R}^d : \|\vw_1\|_1 \le V} \biggl|\frac1n \sum_{i = 1}^n \xi_i \sigma(\vw_1^{\top}\vz_i)\biggr| \nonumber \\
\lesssim& V^{L-1}\sqrt{\frac{d\log n}{n}} \label{eq_j1}, 
\end{align}
where the last inequality follows from the following lemma.
\begin{lemma}
Suppose that $\xi_1, \ldots, \xi_n$ are IID symmetric Bernoulli random variables that are independent with $\vz_i$'s and take values from $\{1,-1\}$. For any given $V$, we have
\begin{align}
\E \sup_{\vw \in \mathbb{R}^d : \|\vw\|_1 \le V} \biggl|\frac1n \sum_{i = 1}^n \xi_i \sigma(\vw^{\top}\vz_i)\biggr| \lesssim \sqrt{\frac{d\log n}{n}}.
\end{align}
\end{lemma}

\begin{proof}
The proof will be based on an $\varepsilon$-net argument together with the union bound.
For any $\varepsilon$, let $W_{\varepsilon} \subset \mathbb{R}^d$ denote the subset
\begin{equation*}
W_{\varepsilon} = \biggl\{\vw = \frac{\varepsilon}{2d}(i_1, i_2, \ldots, i_d) : i_j \in \mathbb{Z}, \|\vw\|_1 \le V \biggr\}.
\end{equation*}
Then, for any $\vw$, there exists some element $\hat{\vw} \in W_{\varepsilon}$ such that 
\begin{align*}
\sup_{\vz \in \mathbb{X}} |\sigma(\vw^{\top}\vz) - \sigma(\hat{\vw}^{\top}\vz)| 
&\le \sup_{\vz} |(\vw^{\top}\vz) - (\hat{\vw}^{\top}\vz)| \\
&\le \sup_{\vz} |(\vw - \hat{\vw})^{\top}\vz| \\
&\le \|\vw - \hat{\vw}\|_1 \sup_{\vz}\|\vz\|_{\infty} \le \varepsilon.
\end{align*}
By Bernstein's Inequality, for any $\vw$, 
\begin{equation*}
\pr\biggl(\biggl|\frac1n \sum_{i = 1}^n \xi_i \sigma(\vw^{\top}\vz_i)\biggr| > t\biggr) \le 2\exp\biggl\{-\frac{nt^2}{2(1+t/3)}\biggr\}.
\end{equation*}
By taking the union bound over $W_{\varepsilon}$, and use the fact that $\log \textrm{card}(W_{\varepsilon}) \lesssim d\log(nd/\varepsilon)$, we obtain 
\begin{equation*}
\sup_{\vw \in \mathbb{R}^d : \|\vw\|_1 \le V} \biggl|\frac1n \sum_{i = 1}^n \xi_i \sigma(\vw^{\top}\vz_i)\biggr| \lesssim \varepsilon + \sqrt{\frac{d}{n}\log \frac{nd}{\varepsilon\delta}},
\end{equation*}
with probability at least $1 - \delta$.
Then the desired result is obtained by taking $\varepsilon = \sqrt{(d\log n)/n}, \delta = n^{-1}$ and using the fact $\bigl|n^{-1} \sum_{i = 1}^n \xi_i \sigma(\vw^{\top}\vz_i)\bigr| \le 1$ as the proof of Theorem~\ref{thm:deep-nets}.
\end{proof}

\end{appendix}




\balance
\bibliography{neural.bib}   

\begin{thebibliography}{10}
\providecommand{\url}[1]{#1}
\csname url@samestyle\endcsname
\providecommand{\newblock}{\relax}
\providecommand{\bibinfo}[2]{#2}
\providecommand{\BIBentrySTDinterwordspacing}{\spaceskip=0pt\relax}
\providecommand{\BIBentryALTinterwordstretchfactor}{4}
\providecommand{\BIBentryALTinterwordspacing}{\spaceskip=\fontdimen2\font plus
\BIBentryALTinterwordstretchfactor\fontdimen3\font minus
  \fontdimen4\font\relax}
\providecommand{\BIBforeignlanguage}[2]{{%
\expandafter\ifx\csname l@#1\endcsname\relax
\typeout{** WARNING: IEEEtran.bst: No hyphenation pattern has been}%
\typeout{** loaded for the language `#1'. Using the pattern for}%
\typeout{** the default language instead.}%
\else
\language=\csname l@#1\endcsname
\fi
#2}}
\providecommand{\BIBdecl}{\relax}
\BIBdecl

\bibitem{cybenko1989approximations}
G.~Cybenko, ``{Approximations by superpositions of a sigmoidal function},''
  \emph{Math. Control Signals Syst.}, vol.~2, pp. 183--192, 1989.

\bibitem{barron1993universal}
A.~R. Barron, ``{Universal approximation bounds for superpositions of a
  sigmoidal function},'' \emph{IEEE Trans. Inf. Theory}, vol.~39, no.~3, pp.
  930--945, 1993.

\bibitem{barron1994approximation}
------, ``{Approximation and estimation bounds for artificial neural
  networks},'' \emph{Mach. Learn.}, vol.~14, no.~1, pp. 115--133, 1994.

\bibitem{Baraniuk}
R.~Baraniuk and R.~Balestriero, ``{A spline theory of deep learning},'' in
  \emph{Proc. ICML}, 2018, pp. 374--383.

\bibitem{golowich2017size}
N.~Golowich, A.~Rakhlin, and O.~Shamir, ``{Size-independent sample complexity
  of neural networks},'' \emph{arXiv Prepr. arXiv1712.06541}, 2017.

\bibitem{barron2019complexity}
A.~R. Barron and J.~M. Klusowski, ``{Complexity, statistical risk, and metric
  entropy of deep nets using total path variation},'' \emph{arXiv Prepr.
  arXiv1902.00800}, 2019.

\bibitem{DingNN}
G.~Li, Y.~Gu, and J.~Ding, ``{The Efficacy of L1 Regularization in Neural
  Networks},'' \emph{arXiv Prepr. arXiv2010.01048}, 2020.

\bibitem{neyshabur2015norm}
B.~Neyshabur, R.~Tomioka, and N.~Srebro, ``{Norm-based capacity control in
  neural networks},'' \emph{Conf. Learn. Theory}, pp. 1376--1401, 2015.

\bibitem{janzamin2015beating}
M.~Janzamin, H.~Sedghi, and A.~Anandkumar, ``{Beating the perils of
  non-convexity: Guaranteed training of neural networks using tensor
  methods},'' \emph{arXiv Prepr. arXiv1506.08473}, 2015.

\bibitem{ge2017learning}
R.~Ge, J.~D. Lee, and T.~Ma, ``{Learning one-hidden-layer neural networks with
  landscape design},'' \emph{arXiv Prepr. arXiv1711.00501}, 2017.

\bibitem{mondelli2018connection}
M.~Mondelli and A.~Montanari, ``{On the connection between learning two-layers
  neural networks and tensor decomposition},'' \emph{arXiv Prepr.
  arXiv1802.07301}, 2018.

\bibitem{schmidt2017nonparametric}
J.~Schmidt-Hieber, ``{Nonparametric regression using deep neural networks with
  ReLU activation function},'' \emph{arXiv Prepr. arXiv1708.06633}, 2017.

\bibitem{bauer2019deep}
B.~Bauer and M.~Kohler, ``{On deep learning as a remedy for the curse of
  dimensionality in nonparametric regression},'' \emph{Ann. Stat.}, vol.~47,
  no.~4, pp. 2261--2285, 2019.

\bibitem{yang1999information}
Y.~Yang and A.~Barron, ``{Information-theoretic determination of minimax rates
  of convergence},'' \emph{Ann. Stat.}, pp. 1564--1599, 1999.

\bibitem{DingOverview}
J.~Ding, V.~Tarokh, and Y.~Yang, ``{Model selection techniques: An overview},''
  \emph{IEEE Signal Process. Mag.}, vol.~35, no.~6, pp. 16--34, 2018.

\bibitem{DingLOL}
J.~Ding, E.~Diao, J.~Zhou, and V.~Tarokh, ``{On Statistical Efficiency in
  Learning},'' \emph{http://jding.org/jie-uploads/2020/05/LoL.pdf}, vol.~65,
  no.~6, pp. 3034--3067, 2020.

\bibitem{ketkar2017introduction}
N.~Ketkar, ``{Introduction to pytorch},'' \emph{Deep Learn. with python}, pp.
  195--208, 2017.

\bibitem{Welsch1977}
R.~E. Welsch, ``{Robust regression using iteratively reweighted
  least-squares},'' \emph{Commun. Stat. - Theory Methods}, 1977.

\bibitem{pedregosa2017consistency}
F.~Pedregosa, F.~Bach, and A.~Gramfort, ``{On the consistency of ordinal
  regression methods},'' \emph{J. Mach. Learn. Res.}, vol.~18, no.~1, pp.
  1769--1803, 2017.

\bibitem{zhao2016loss}
H.~Zhao, O.~Gallo, I.~Frosio, and J.~Kautz, ``{Loss functions for image
  restoration with neural networks},'' \emph{IEEE Trans. Comput.}, vol.~3,
  no.~1, pp. 47--57, 2016.

\bibitem{abadi2016tensorflow}
R.~T. Google, ``{Tensorflow: A system for large-scale machine learning},''
  \emph{Proc. 12th Symp. Oper. Syst. Des. Implement.}, pp. 265--283, 2016.

\bibitem{gulli2017deep}
A.~Gulli and S.~Pal, \emph{{Deep Learning with Keras}}.\hskip 1em plus 0.5em
  minus 0.4em\relax Packt Publishing Ltd, 2017.

\bibitem{zhang2016understanding}
C.~Zhang, S.~Bengio, M.~Hardt, B.~Recht, and O.~Vinyals, ``{Understanding deep
  learning requires rethinking generalization},'' \emph{arXiv Prepr.
  arXiv1611.03530}, 2016.

\bibitem{van1994bracketing}
A.~van~der Vaart, ``Bracketing smooth functions,'' \emph{Stochastic Processes
  and their Applications}, vol.~52, no.~1, pp. 93--105, 1994.

\end{thebibliography}
\bibliographystyle{IEEEtran}

\end{document}